\documentclass[letterpaper, 12pt]{amsart}
\usepackage{graphicx} 
\pdfoutput=1

\usepackage{orcidlink,thumbpdf,lmodern}
\usepackage{amssymb}
\usepackage[utf8]{inputenc}
\usepackage{bbm,tikz}
\RequirePackage{amsthm,amsmath,amsfonts,amssymb}
\RequirePackage{graphicx,url,color}
\usepackage{cleveref}
\usepackage{comment}
\usepackage{soul}
\usepackage{algpseudocode}
\usepackage{algorithm,color,xcolor}
\usepackage{enumitem,amsmath}   
\usepackage{dsfont}
\usepackage{float}

\usepackage[a4paper]{geometry}
\newtheorem{theorem}{Theorem}[section]
\newtheorem{lemma}[theorem]{Lemma}
\newtheorem{definition}[theorem]{Definition}
\newtheorem*{theorem*}{Theorem}
\newtheorem{example}[theorem]{Example}

\newtheorem{corollary}[theorem]{Corollary}

\newcommand{\Trop}{\text{Trop}}

\DeclareMathOperator*{\argmin}{arg\,min}
\DeclareMathOperator*{\argmax}{arg\,max}

\usepackage{framed}
\usepackage{tikz}
\usetikzlibrary{arrows, positioning, automata}
\usepackage{array}
\usepackage{multirow}
\usepackage{caption}
\newcommand{\trt}{{\mathbb{R}^{d+1}\slash \mathbb{R} \mathbf{1}}}
\newcommand{\bis}{{\mathrm{bis}}}

\newcommand{\Sijkl}[0]{S_{i,j,k,\ell}}

\begin{document}

\title{Tropical Bisectors and Carlini-Wagner Attacks}
\author[Grindstaff, Lindberg, Schkoda, Sorea, Yoshida]{Gillian Grindstaff  \and Julia Lindberg \and Daniela Schkoda \and Miruna-Stefana Sorea \and Ruriko Yoshida}
\date{}

\begin{abstract}
    Pasque et al. showed that using a tropical symmetric metric as an activation function in the last layer can improve the robustness of convolutional neural networks (CNNs) against state-of-the-art attacks, including the Carlini-Wagner attack. This improvement occurs when the attacks are not specifically adapted to the non-differentiability of the tropical layer.
  Moreover, they showed that the decision boundary of a tropical CNN is defined by tropical bisectors.  
  In this paper, we explore the combinatorics of tropical bisectors and analyze how the tropical embedding layer enhances robustness against Carlini-Wagner attacks.
  We prove an upper bound on the number of linear segments the decision boundary of a tropical CNN can have. 
  We then propose a refined version of the Carlini-Wagner attack, 
  specifically tailored for the tropical architecture. 
  Computational experiments with MNIST and LeNet5 showcase our attacks improved success rate.
\end{abstract}
\maketitle
\section[Introduction]{Introduction} \label{sec:intro}
Artificial neural networks demonstrate exceptional capability in the fields of natural language processing, computer vision, and genetics. However, they have revealed exceptional vulnerability to adversarial attacks \cite{carlini2017evaluating}.  
As neural networks become prevalent in critical applications such as autonomous driving, healthcare, and cybersecurity, the development of adversarial defense methodologies is central to the reliability and ultimate success of these efforts. In 2024, Pasque et al. introduced a simple and easy to implement convolutional neural network (CNN) robust against state-of-the-art adversarial attacks \cite{Pasque2024}, such as the Carlini-Wagner attack \cite{carlini2017evaluating}, if these attacks are employed with vanilla gradient descent. 
These CNNs feature {\em tropical embedding layers} and their decision boundary is a {\em tropical bisector}, which is the locus of points which are equidistant to a given set in terms of the tropical symmetric distance.

A classifier neural network is a function  
\[
F_\theta: \mathbb{R}^d \to [C] := \{1, \ldots , C\}
\]
that maps an input \( x \) to a predicted class label \( \hat{c} \). In our context of image classification the input domain is $[0,1]^d$ instead of $\mathbb{R}^d$. The model parameters $ \theta$  are found by minimizing the classification error on the training data and remain fixed afterwards. Correspondingly, we simply write $F$ in the following. 

Adversarial machine learning refers to machine learning in adversarial environments, where attacks play a crucial role. There are two main types of attacks:  data poisoning and evasion.  
Data poisoning occurs when attackers manipulate the training set to deceive the model, while an evasion attack involves attackers altering the test set.  
In an evasion attack, one is either using a white box or a black box attack.  
A white-box attack is a type of attack where attackers have complete knowledge of the machine learning model, including access to parameters, hyperparameters, gradients, network architecture, and more. In contrast, a black-box attack occurs when attackers do not have access to the model's parameters, gradients, or architecture.

An \textit{adversarial attack} is a small perturbation \( \delta \) added to an input \( x \), causing the network to misclassify the new data point:  

\[
x' = x + \delta, \quad F(x') \neq c^*,
\]
where $c^*$ is the true label of $x'$.

Carlini and Wagner introduced  adversarial perturbation techniques with Euclidean metrics \cite{carlini2017evaluating} and they showed that these
attacks are able to successfully evade the defensive distillation method \cite{distlation}. The Carlini-Wagner attack is formulated as the solution to an optimization problem which aims to find the smallest perturbation 
with respect to a given $\ell_p$ norm 
so that this perturbation pushes the input $x$ across the decision boundary of the Voronoi region of a target class $t \in [C]$.  The optimization problem can be formulated as the following:  
\begin{align*}
   \text{minimize} &\quad \| \delta \|_p + \lambda \cdot f(x + \delta) \\
   \text{subject to} &\quad x + \delta \in [0, 1]^d,
\end{align*}
where $f: \mathbb{R}^d \rightarrow \mathbb{R}$ such that $f(x + \delta) \leq 0$  
if and only if $x + \delta$ is classified as being from class $\tau \in K$, and a constant $\lambda > 0$.
Gradient descent is used to solve the optimization problem above.
 
The purpose of this paper is to further investigate tropical CNNs and their robustness to adversarial attacks. To better understand how the tropical embedding layer improves robustness against adversarial attacks, we begin by investigating the combinatorics of tropical bisectors of two points. We give an upper bound on the number of linear components such tropical bisectors can have. We then consider how gradients of 
the Carlini-Wagner attack interact with tropical bisectors and we provide evidence as to why tropical CNNs are resilient to 
 Carlini-Wagner attacks.  We conclude by proposing an adversarial attack effective against tropical CNNs and we provide computational experiments with MNIST and LeNet5.  The code for this project is available at \url{https://github.com/DanielaSchkoda/TropicalCNN}.

\section[Background]{Background} \label{sec:back}

\subsection{Convolutional neural networks}
Neural networks are generalizations of linear models which consists of three types of layers: input layers, hidden layers, and output layers.  An input layer is a layer where one provides an observation. Therefore the number of neurons in this layer is the dimension of the input vector.   
An affine linear combination of outputs from the previous layer is then fed into each neuron in a hidden layer. 
Each neuron then outputs the result of an activation function.  There can be many hidden layers and each hidden layer can have a different number of neurons. The weights needed to form each affine linear combination of neurons used as input in a hidden layer are parameters and we estimate them by gradient descents via {\em back propagation}. The output from the final hidden layer is then fed into an activation function at each neuron in the output layer.  An activation function in the output layer depends on the type of supervised learning task one wishes to do.  For binomial classification, one typically uses the sigmoid function. For multinominal classification, the softmax function is a common choice. For a tropical CNN, we use tropical embedding layer as the output layer. 

Convolutional Neural Networks (CNNs) are deep neural networks comprised of an input layer, convolutional layer, activation layer,  pooling layer, flattening process, fully connected layers and an output layer. The goal of a convolutional layer is to extract a feature (feature engineering) from the input 
via a set of learnable filters known as kernels.  
Usually these filters are smaller matrices that slide over the input and take the dot product between filter  parameters and 
subsets of the input image. In the activation layer, a ReLU activating function is applied to the output from the convolutional layer. Next the pooling layer extracts statistics from a subset of outputs from the activation layer. An output from this sequence of layers 
is a tensor. Therefore, in the flattening process, one flattens the output tensor from the pooling layer to a vector. This vector is then fed to a fully connected neural network.  See \cite{CNNBook} for details on CNNs.

\subsection{Tropical CNNs}
For readers unfamiliar with tropical geometry, further background can be found in \cite{MS}.
First, we consider the 
tropical semiring $(\,\mathbb{R} \cup \{-\infty\},\oplus,\odot)\,$ with the max-plus algebra (the tropical arithmetic operations of addition and multiplication) defined as:
$$a \oplus b := \max\{a, b\}, ~~~~ a \odot b := a + b $$
for any $a, b \in \mathbb{R}\cup\{-\infty\}$ where $-\infty$ is the identity element for the tropical addition operation $\oplus$, and $0$ is the identity element for the tropical multiplication operation $\odot$.

The tropical projective torus is defined as
\[
\mathbb{R}^{d+1}/\mathbb{R}{\bf 1}:= \left\{x \in \mathbb{R}^{d+1}\mid x:=(x_1, \ldots , x_{d+1}) = (x_1+c, \ldots , x_d+c), \, \forall c \in \mathbb{R}\right\},
\]
where ${\bf 1} = (1, \ldots , 1) \in \mathbb{R}^{d+1}$. This means that $\mathbb{R}^{d+1}/\mathbb{R}{\bf 1}$ is invariant under the translation by $(c, c, \ldots , c) \in \mathbb{R}^{d+1}$, i.e., for any $x = (x_1, x_2, \ldots , x_{d+1}) \in \mathbb{R}^{d+1}/\mathbb{R}{\bf 1}$,
\[
x = (x_1, x_2, \ldots , x_{d+1}) = (0, x_2 - x_1, \ldots , x_{d+1} - x_1) \in \mathbb{R}^{d+1}/\mathbb{R}{\bf 1}.
\]
This means that the tropical projective torus $\mathbb{R}^{d+1}/\mathbb{R}{\bf 1}$ is isomorphic to $\mathbb{R}^{d}$.

\begin{definition}[Tropical Metric]\label{def:tropicalMetric}
    The {\em tropical metric} is defined for any $x = (x_1, \ldots , x_{d+1}), y = (y_1, \ldots , y_{d+1}) \in \mathbb{R}^{d+1}/\mathbb{R}{\bf 1}$ as
    \[
    d_{\rm tr}(x, y) = \max_{i\in \{1, \ldots , {d+1}\} }\{x_i - y_i\} - \min_{i\in \{1, \ldots , {d+1}\} }\{x_i - y_i\}.
    \]
\end{definition}

\begin{definition}[Tropical Embedding Layer]\label{def:trop_embedding_layer}
A {\em tropical embedding layer} takes a vector $x \in \mathbb{R}^{d+1}/\mathbb{R}\mathbf{1}$ as input, and the activation of the $j$-th neuron in the embedding layer some $l \in \{1, \ldots , L-1\}$ is
\begin{equation} \label{eq:embed}
z_j = \max_i(x_i + w^{(l)}_{ji}) - \min_i(x_i + w^{(l)}_{ji}) = d_{\rm tr}(-{\bf w}^{(l)}_{j}, x).
\end{equation}
where ${\bf w}^{(l)}_{j} = (w^{(l)}_{j1}, \ldots , w^{(l)}_{j{d+1}})$.
\end{definition}

Here we focus on classifications with a finite positive integer $K \geq 2$ classes via a tropical convolutional neural network (CNN) which is naturally robust against black-box and white-box adversarial attacks without adversarial training \cite{Pasque2024} shown in  \Cref{fig:tropicalCNN}.  Let $L$ be the number of layers in a CNN and denote by $f^i, i = 1, \ldots, L-1$  a sequence of functions representing hidden layers in the architecture. 
Suppose $f^{L-1}: \mathbb{R}^{n_1\times n_2\times n_3} \to \mathbb{R}^K/\mathbb{R}^1$ is the map of a convolutional neural network classifier. Then $f^{L-1}(x)$, with input data $x \in \mathbb{R}^{n_1\times n_2\times n_3}$ of $K$ classes, is the output of the network. We then embed the output $f^{L-1}(x)$ into the tropical projective torus using coordinates $z_j$ as in (\ref{eq:embed}):
\[
    f_j^L(x) = \max_i\{f_i^{L-1}(x) + w_{ji}^{(L)}\} - \min_i\{f_{ji}^{L-1}(x)+ w_{ji}^{(L)}\} \\
    = d_{tr}(-{\bf w}_j^{(L)}, x), \  \ j \in {1, \dots, K}.
\]
We show in the following section that the weights, ${\bf w}_j^{(L)}$, train towards {\em Fermat-Weber points} associated with the $k$ classes. See \cite{BSYM}  for details on Fermat-Weber points and using gradient descent to find them.  

Each dimension of the output $f^L(x) \in \mathbb{R}^k$ is the tropical distance between the input and Fermat-Weber points of each class. We then classify the input with a softmin function:
\begin{equation}\label{eq:tropSoftMax}
    \max_{j \in {1, \ldots, K}} \frac{{e^{-d_{tr}(-{\bf w}_j^{(L)},x)}}}{\sum_{i=1}^k e^{-d_{tr}(-{\bf w}_i^{(L)},x)}}.
\end{equation}
Let 
\[
    p_k=  \frac{{e^{-d_{tr}(-{\bf w}_j^{(L)},x)}}}{\sum_{i=1}^k e^{-d_{tr}(-{\bf w}_i^{(L)},x)}}
\]
for an estimated probability for classifying as the class $k \in [K]$. 
\begin{figure}
    \centering
    \includegraphics[width=\textwidth]{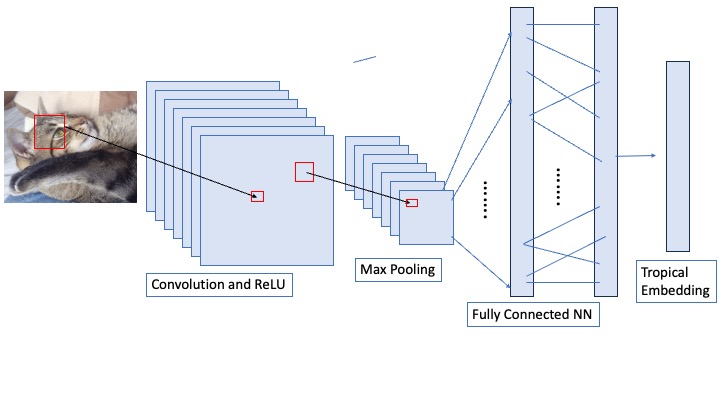}
    
    \vskip -0.5 in
    \caption{Tropical CNN architecture from \cite{Pasque2024}.}\label{fig:tropicalCNN}
\end{figure}
\begin{definition}[Max-tropical Hyperplane~\cite{ETC}]
    \label{def:trop_hyp} 
    A max-tropical hyperplane $H^{\max}_{\alpha}$ is the set of points $u = (u_1, \ldots , u_d)\in \mathbb R^{d+1} \!/\mathbb R {\bf 1}$ satisfying that  
    \begin{equation}\label{eq:trop_hyp}
    \max_{i \in \{1, \ldots , d\}} \{u_i + \alpha_i\}
    \end{equation} 
    is attained at least twice, where $\alpha:=(\alpha_1, \ldots, \alpha_d)\in \mathbb R^{d+1} \!/\mathbb R {\bf 1}$.
\end{definition}

\begin{definition}[Min-tropical Hyperplane~\cite{ETC}]
    \label{def:min_trop_hyp} 
    A min-tropical hyperplane $H^{\min}_{\alpha}$ is the set of points $u = (u_1, \ldots , u_d)\in \mathbb R^{d+1} \!/\mathbb R {\bf 1}$ satisfying that   
    \begin{equation}\label{eq:min_trop_hyp}
    \min_{i \in \{1, \ldots , d\}} \{u_i + \alpha_i\}
    \end{equation} 
    is attained at least twice, where $\alpha:=(\alpha_1, \ldots, \alpha_d)\in \mathbb R^{d+1} \!/\mathbb R {\bf 1}$.
\end{definition}

\begin{definition}[Max-tropical Sectors from Section 5.5 in \cite{ETC}]\label{def:sector}
    For $i\in[d]$, the $i{\rm-th}$ {\em open sector} of $H^{\max}_\alpha$ is defined as
    \begin{equation}\label{eq:max_open_sector}
        S^{\max,i}_\alpha := \{\mathbf{u}\in \mathbb{R}^{d+1}/\mathbb{R}\mathbf{1}\,|\, \alpha_i + u_i > \alpha_j +u_j, \forall j\neq i\},
    \end{equation}
    \noindent and the $i{\rm-th}$ {\em closed sector} of $H^{\max}_\alpha$ is defined as
    \begin{equation}\label{eq:max_closed_sector}
        \overline{S}^{\max,i}_\alpha := \{\mathbf{u}\in \mathbb{R}^{d+1}/\mathbb{R}\mathbf{1}\,|\, \alpha_i + u_i \geq \alpha_j +u_j, \forall j\neq i\}.
    \end{equation}
\end{definition}

\begin{definition}[Min-tropical Sectors]\label{def:min_sector}
    For $i\in[d]$, the $i{\rm-th}$ {\em open sector} of $H^{\min}_\alpha$ is defined as
    \begin{equation}\label{eq:min_open_sector}
        S^{\min,i}_\alpha := \{\mathbf{u}\in \mathbb{R}^{d+1}/\mathbb{R}\mathbf{1}\,|\, \alpha_i + u_i < \alpha_j +u_j, \forall j\neq i\},
    \end{equation}
    \noindent and the $i{\rm-th}$ {\em closed sector} of $H^{\min}_\alpha$ is defined as
    \begin{equation}\label{eq:min_closed_sector}
        \overline{S}^{\min,i}_\alpha := \{\mathbf{u}\in \mathbb{R}^{d+1}/\mathbb{R}\mathbf{1}\,|\, \alpha_i + u_i \leq \alpha_j +u_j, \forall j\neq i\}.
    \end{equation}
\end{definition}

Consider ${S}^{\max, i}_{-x}$ and ${S}^{\min, i}_{-x}$ are the $i$-th open sectors of the max/min-tropical hyperplanes with apex at $x$ as defined in Equations \ref{eq:max_closed_sector} and \ref{eq:min_closed_sector}.
Then for any point $u$ in $\mathbb R^{d+1} \!/\mathbb R {\bf 1}$, we can identify the indices $k$ and $l$ 
where $1\leq k\neq l \leq d$ correspond to the $l$-th open min-sector and the $k$-th open max-sector 
for which $u \in {S}^{\max, k}_{-x} \cap {S}^{\min, l}_{-x}$.
To simplify for our purposes, we use the Kronecker's Delta:
\begin{equation}
\delta_{ij} = \begin{cases}
    1 & \mbox{if } i = j,\\
    0 & \mbox{otherwise}
\end{cases}
\end{equation}
for indicating the indices of a given vector.  

\begin{definition}[Indicators for Open Sectors]
For a given point $x$ in $\mathbb R^{d+1} \!/\mathbb R {\bf 1}$, let $k$ and $l$ be the indices for which $u \in {S}^{\max, k}_{-x} \cap {S}^{\min, l}_{-x}$.
Then, the max/min-tropical indicator vectors for open sectors are defined by
\[
\left({G}^{\max}_{-x}(u)\right)_i = \delta_{ik},\qquad
\left({G}^{\min}_{-x}(u)\right)_i = \delta_{il},\qquad
\textrm{for } i\in [d].
\]
That is, ${G}^{\max}_{-x}(u)$ is the $k$-th unit vector and ${G}^{\min}_{-x}(u)$ is the $l$-th unit vector.
\end{definition}

\begin{lemma}[\cite{BSYM}]
For any two points $x,u\in \mathbb R^{d+1} \!/\mathbb R {\bf 1}$, the gradient at $u$ of the tropical distance between $x$ and $u$ is given by
    \begin{equation}
        \frac{\partial d_{\rm tr}(u,x)}{\partial u} = {G}^{\max}_{-x}(u) - {G}^{\min}_{-x}(u) = \left(\delta_{ik} - \delta_{il}\mid u\in {S}^{\max, k}_{-x} \cap {S}^{\min, l}_{-x}\right)
    \end{equation}
if there are no ties in $(u-x)$, implying that the min- and max-sectors are uniquely identifiable, that is, the point $u$ is inside of open sectors and not on the boundary of $H_{-x}$.
\end{lemma}

Now, we have the Jacobian of the softmax as the following.  Let $z_i = d_{\rm tr}(w_i, x)$.
\begin{lemma}[\cite{miranda2017softmax}]\label{tm:grad1}
We have:
    \begin{equation}
        \frac{\partial p_k}{\partial z_j} = p_k \left(\delta_{kj} - p_j\right).
    \end{equation}
\end{lemma}

Also note that we have
\[
\frac{\partial (-\log p_k)}{\partial p_k} = -\frac{1}{p_k}\frac{\partial p_k}{\partial z_j}.
\]

With \Cref{tm:grad1}, we have an explicit formula for the gradient for the tropical embedding layer located at the output layer of the tropical CNN as follows: 
\begin{theorem}[\cite{BSYM}] Suppose we have $L \geq 2$ hidden layers between the input layer and the output layer in a tropical CNN, i.e., the $L$th hidden layer is the tropical embedding layer. Let $K$ be the number of classes in the response variable, so that the tropical embedding layer has incoming edge weights $w_j \in \mathbb{R}^K/\mathbb{R}{\bf 1}$. Then for the $j$th neuron in the tropical embedding layer, we have a gradient: 
\begin{equation}
    \frac{\partial p_k}{\partial w_j} = -\left(\delta_{kj} - p_j\right)\left(\delta_{ik} - \delta_{il}\mid w_j\in {S}^{\max, k}_{-x} \cap {S}^{\min, l}_{-x}\right).
\end{equation}
\end{theorem}

\section{Combinatorics of tropical bisectors}
This section explores the geometry and combinatorics of the decision boundary of a tropical CNN. For any network with a final tropical embedding layer as described in \Cref{def:trop_embedding_layer}, the decision boundary consists locally of sets of tropically equidistant points, called tropical bisectors.
Given two points $a,b \in \trt$, their \emph{tropical bisector} $\bis(a,b)$ is defined:
\[ \bis(a,b) = \{x \in \trt \ : \ \text{dist}_{\mathrm{tr}}(x,a) = \text{dist}_{\mathrm{tr}}(x,b) \}. \]

We will see in this section that $\bis(a,b)$ consists of a finite number of connected linear segments. Let $\mathcal{C}(a,b) $ be the number of such linear pieces of $\bis(a,b)$, a marker of the complexity of the decision boundary.

As outlined in \cite[Section 4]{Criado_2021}, the distance condition $\text{dist}_{\mathrm{tr}}(x,a) = \text{dist}_{\mathrm{tr}}(x,b)$ is equivalent to finding $x \in \trt$ such that 
\begin{align} \max_{i,j \in [d+1]} \{ x_i - a_i -x_j + a_j \} = \max_{k,\ell \in [d+1]} \{ x_k - b_k - x_\ell + b_\ell \}. \label{eq: trop bis}
\end{align}

For fixed $a,b \in \trt$, $x$ satisfies the above equation if and only if there are $i,j,k,\ell$ (i.e. indices achieving the maxima) such that:
\begin{align}\label{eq: bis sys gen} 
    x_i - a_i - x_j + a_j &= x_k - b_k -x_\ell + b_\ell \notag\\
    x_i - a_i - x_j + a_j &\geq x_m - a_m - x_n + a_n  &&\forall m,n \in [d+1] \tag{$A_{ijk\ell}$} \\
    x_k - b_k - x_\ell + b_\ell &\geq x_m - b_m - x_n + b_n &&\forall m,n \in [d+1]. \notag 
\end{align}
To study $\bis(a,b),$ we decompose it according to these critical indices:
\begin{definition}
  For each $i,j,k,\ell\in [d+1]$, $i \neq j$ and $k \neq \ell$, we define $\Sijkl(a,b)$ to be the set of solutions to the linear system \eqref{eq: bis sys gen}. 
\end{definition}

 We note as a consequence that \[\bis(a,b) = \bigcup_{i,j,k,\ell}\Sijkl(a,b)\]
and that $\Sijkl (a,b)$ is either empty (if the system is infeasible) or contains an equidistant segment satisfying \eqref{eq: trop bis} with indices $i,j, k,\ell$. For generic $a,b$, then, $\mathcal{C}(a,b)$ also counts the number of non-empty $\Sijkl(a,b)$. Furthermore, in \cite{icking1995convex} and \cite{Criado_2021} it was shown that if $a$ and $b$ are sufficiently general with respect to the facets defining the tropical unit ball, then $\bis(a,b)$ is homeomorphic to $\mathbb{R}^{d-1}$. A corollary of this result is that $\bis(a,b)$ is connected.

Since $\bis(a,b)$ is invariant under translation, we assume without loss of generality that $a = 0 \in \trt$ and denote $\bis(b) := \bis(0,b)$, $\Sijkl(b) := \Sijkl(0,b)$, and $\mathcal{C}(b):= \mathcal{C}(0,b)$ for simplicity. 

Therefore, our question now reduces to studying the feasibility of $\Sijkl(b)$, defined by the slightly simpler equations
\begin{align}\label{eq: bis sys} 
    x_i  - x_j  &= x_k - b_k -x_\ell + b_\ell \notag\\
    x_i  - x_j  &\geq x_m  - x_n   &&\forall m,n \in [d+1] \tag{$A^*_{ijkl}$} \\
    x_k - b_k - x_\ell + b_\ell &\geq x_m - b_m - x_n + b_n &&\forall m,n \in [d+1] \notag
\end{align}
where $i \neq j$ and $k \neq \ell$.

\begin{example}\label{eg:secbiseceg}
    Consider $a = (0, 0, 0), \, b = (1, 2, 0) \in \mathbb{R}^{3}/\mathbb{R}{\bf 1}$.
     For indices $i = 1, \, j = 3, \, k = 1, \, \ell = 2$ \eqref{eq: bis sys gen} is feasible. This corresponds to the maximum in \eqref{eq: trop bis} being achieved at these same indices.  Therefore we have 
    \[
    \bis(a,b) = \left\{(x_1, 1, 0): x_1 \in \mathbb{R}\right\}.
    \]
    
    \begin{figure}[h!]
        \centering
        \includegraphics[width=0.7\textwidth]{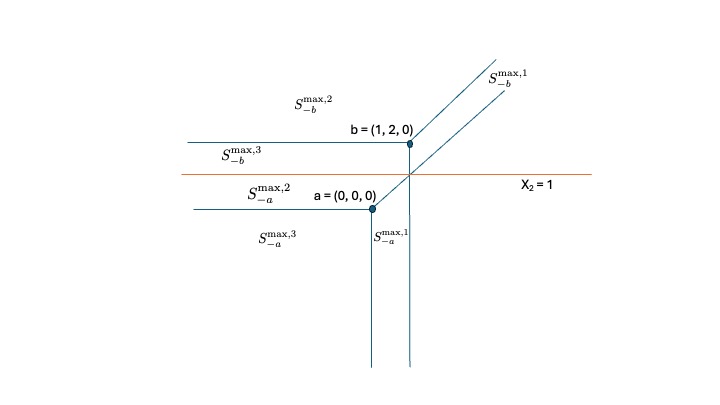}
        \caption{Blue lines represent tropical hyperplanes $H_{-a}^{\max}$ and $H_{-b}^{\max}$ where $a = (0, 0, 0)$ and $b=(1, 2, 0)$. The red line represents the bisector $\bis(a, b)$. }
        \label{fig:secbiseceg}
    \end{figure}

    \Cref{fig:secbiseceg} shows the relation between sectors $S_{-a}^{\max, 1}$, $S_{-a}^{\max, 2}$, $ S_{-a}^{\max, 3}$, $S_{-b}^{\max, 1}$, $ S_{-b}^{\max, 2}$, $ S_{-b}^{\max, 3}$, and the bisector $\bis(a, b)$.
\end{example}

Throughout this section, we assume that $b$ is in general position and all coordinates of $b$ are distinct (i.e. $b_i \neq b_j$ for any $i,j \in [d+1]$). Observe that the set of $b \in \trt$ satisfying these assumptions lie on the complement of a set with Lebesgue measure zero.

The following lemmas show that the existence of points in $\Sijkl (b)$ depends on the components of $b$ in restricted and sometimes mutually exclusive ways. We use these facts to bound $\mathcal{C}(b)$ in \Cref{lem:last-bound}.

\begin{lemma}\label{lem: bis non empty bijkl order}
    Suppose $S_{i,j,k,\ell}(b) \neq \emptyset$, then $b_i \geq b_j$, $b_\ell \geq b_k$, and $b_k \leq b_i$.
\end{lemma}
\begin{proof}
    $S_{i,j,k,\ell}(b) \neq \emptyset $ implies that there exists $x \in \trt$ such that \eqref{eq: bis sys} holds for all $m,n\in [d+1].$ Taking various $m,n$ values for ${i,j,k,\ell}$ we see that
    \begin{align}
        x_i - x_j &= x_k - b_k - x_\ell + b_\ell \label{lem1:eq1} \\
        x_i - x_j &\geq x_k - x_\ell \label{lem1:eq2}  \\
        x_k - b_k - x_\ell + b_\ell &\geq x_i -b_i - x_j + b_j. \label{lem1:eq3}  \\
        x_i - x_j &\geq x_k - x_j \label{lem1:eq4}\\
        x_k - x_\ell - (b_k - b_{\ell}) &\geq x_i - x_\ell - (b_i - b_\ell). \label{lem1:eq5}
    \end{align}
    
    Equations \eqref{lem1:eq1} and \eqref{lem1:eq3} imply that $x_i - x_j \geq x_i -b_i - x_j + b_j$ or that $0 \geq -b_i + b_j$, i.e. $b_i \geq b_j$. Similarly, \eqref{lem1:eq1} and \eqref{lem1:eq2} imply that $x_k - b_k - x_\ell + b_\ell \geq x_k - x_\ell$ which gives that $-b_k + b_\ell \geq 0$ or $b_\ell \geq b_k$. Equation \eqref{lem1:eq5} gives $b_k \leq b_i + x_k - x_i$ and \eqref{lem1:eq4} gives that $x_i \geq x_k$, implying $b_i \geq b_k.$
\end{proof}

In particular, the contrapositive of \Cref{lem: bis non empty bijkl order} gives that if $b_i < b_j$, $b_\ell < b_k$, or $b_i < b_k$ then $S_{i,j,k,l}(b) = \emptyset$. As an immediate corollary, we have that if $S_{i,j,k,\ell}(b) \neq \emptyset$ then $S_{j,i,m,n} = \emptyset$ and $S_{m,n,\ell,k} = \emptyset$ for all $m,n \in [d+1]$. If in addition, $i\neq k,$ then $S_{k,m,i,n}= \emptyset$ as well. 

We now refine \Cref{lem: bis non empty bijkl order} based on the maximum and minimum values of $b \in \trt $.

\begin{lemma}\label{lem:max and min} 
Let $i,j,k,\ell \in [d+1]$ with $i \neq j$ and $k \neq \ell$.
\begin{enumerate} 
    \item[(a)] If $\Sijkl(b)\neq \emptyset$ with $i=\ell,$ then 
    $b_i = \max\{b\}$ and $b_j \geq 2b_k - b_i$. 
    \item[(b)]  If $\Sijkl(b)\neq \emptyset$ with $j=k$, then  $b_j = \min \{b\}$ and $b_j \leq 2 b_\ell - b_i$.
\end{enumerate}
In particular, $S_{i,j,j,i} \neq \emptyset$ if and only if $b_i = \max\{b\}$ and $b_j = \min \{b\}$.
\end{lemma}
\begin{proof}
    First suppose $i=\ell$ and $S_{i,j,k,\ell}(b)= S_{i,j,k,i}(b) \neq \emptyset$. Then for any $m \in[d+1]$, there exists some $x \in \trt$ such that 
    \begin{align}
        x_i - x_j &\geq x_m - x_j \label{eq1} \\
        x_k - x_i - (b_k - b_i) &\geq x_k - x_m - (b_k - b_m) \label{eq2}
    \end{align}
    This gives
    \begin{align*}
        b_m &\leq (x_k - x_i) - (b_k - b_i) + (x_m - x_k) + b_k = b_i + x_m - x_i \leq b_i
    \end{align*}
    Equation \eqref{eq2} gives the first inequality and \eqref{eq1} gives the last. In particular, this shows that $b_i = \max \{b\}$. 

    To show that $b_j \geq 2b_k - b_i$, observe that $S_{i,j,k,i}(b) \neq \emptyset$ implies that there exists $x \in \trt$ satisfying
    \begin{align}
        x_i - x_j &= x_k - x_i - (b_k - b_i) \label{eq: sijki 1}\\
        x_i - x_j &\geq x_k - x_j \label{eq: sijki 2} \\
        x_k - x_i - (b_k - b_i) &\geq x_j - x_i - (b_j - b_i) \label{eq: sijki 3}.
    \end{align}
    These inequalities then give:
    \begin{align*}
        b_j &\geq b_i + (x_j - x_i) + (x_i - x_k) + (b_k - b_i) = 2b_k - b_i + 2(x_i - x_k) \geq 2b_k - b_i
    \end{align*}
    where the first inequality is from \eqref{eq: sijki 3}, the second is from \eqref{eq: sijki 1} and the third is from \eqref{eq: sijki 2}.

    To show part (b), suppose $j=k$ and $\Sijkl(b) = S_{i,j,j,\ell}(b) \neq \emptyset$ and consider $x_m$ for any $m \in [d+1]$. Then 
    \begin{align}
        x_i - x_j &\geq x_i - x_m \label{eq3} \\
        x_j - x_\ell - (b_j - b_\ell) &\geq x_m - x_\ell - (b_m - b_\ell) \label{eq4}
    \end{align}
    This gives
    \begin{align*}
        b_j &\leq b_m + x_j - x_m \leq b_m
    \end{align*}
    where the first inequality is by \eqref{eq3} and the second is by \eqref{eq4}. In particular, this shows that $b_j = \min\{b\}$. 

    To show that $b_j \leq 2 b_\ell - b_i$, observe that there is some $x \in S_{i,j,j,\ell}(b)$ that satisfies
    \begin{align}
        x_i - x_j &= x_j - x_\ell - (b_j - b_\ell) \label{eq:sijjk 1}\\
        x_i - x_j &\geq x_i - x_\ell \label{eq:sijjk 2}\\
        x_j - x_\ell - (b_j - b_\ell ) &\geq x_j - x_i - (b_j - b_i)\label{eq:sijjk 3}.
    \end{align}
    These inequalities then give
    \begin{align*}
        b_j &= b_\ell + (x_j - x_\ell) + (x_j - x_i) \\
        &\leq b_\ell + 2(x_j - x_\ell) - (b_j - b_\ell) + (b_j - b_i) \\
        &= 2 b_\ell - b_i + 2(x_j - x_\ell) \\
        &\leq 2 b_\ell - b_i,
    \end{align*}
    where the first line is from \eqref{eq:sijjk 1}, the second line is from \eqref{eq:sijjk 3} and the final line is by \eqref{eq:sijjk 2}.

    Finally, it remains to show that if $b_i = \max \{b\}$ and $b_j = \min \{b\}$ then $S_{i,j,j,i} \neq \emptyset$. Suppose that $b_i = \argmax\{b\}$ and $b_j = \argmin \{b\}$. We claim that the point $x = \frac{1}{2} b$ is in $S_{i,j,j,i}(b)$. One can directly verify that the equality constraint $x_i - x_j = x_j - b_j - x_i + b_i$ is met. Next, observe that $b_i - b_j \geq b_m - b_n$ for all $m,n \in [d+1]$. With this we see both inequality constraints $x_i - x_j \geq x_m - x_n$ and $x_j - b_j - x_i + b_i \geq x_m - b_m - x_n + b_n$ is met for all $m,n \in [d+1]$. Therefore, $x = \frac{1}{2} b \in S_{i,j,j,i}(b)$.
\end{proof}

To this point, we have shown how the feasibility of $\Sijkl (b)$ implies certain relationships among the entries of $b \in \trt$. We now show that there is symmetry among the various $\Sijkl (b)$ components.

\begin{lemma}\label{lem: pairs of nonempty}
   $S_{\ell,k,j,i}$ is an invertible affine transformation of $\Sijkl$. In particular,  $S_{i,j,k,\ell}(b) \neq \emptyset$ if and only if $S_{\ell,k,j,i}(b) \neq \emptyset$. Further, if $S_{i,j,k,\ell} (b) \neq \emptyset$ then $b_j \leq b_\ell$. 
\end{lemma}
\begin{proof}
    Suppose $S_{i,j,k,\ell}(b) \neq \emptyset$. Then there exists an $x \in \trt$ such that 
    \begin{align}
    x_i  - x_j  &= x_k - b_k -x_\ell + b_\ell \label{eq11}\\
    x_i  - x_j  &\geq x_m  - x_n   &&\forall m,n \in [d+1] \label{eq21} \\
    x_k - b_k - x_\ell + b_\ell &\geq x_m - b_m - x_n + b_n &&\forall m,n \in [d+1] \label{eq31}
\end{align}
Consider the affine transformation 
\[ y_\ell = x_i, \  y_k = x_j, \  y_j = x_k - b_k + b_j, \  y_i = x_\ell - b_\ell + b_i, \   y_n = x_n,   \ \forall n \in [d+1] \backslash \{i,j,k,\ell\}. \]
We want to show that  $y \in \trt$ satisfies
\begin{align}
    y_\ell  - y_k  &= y_j - b_j - y_i + b_i \label{eq41}\\
    y_\ell - y_k &\geq y_m - y_n   &&\forall m,n \in [d+1]\label{eq51} \\
    y_j - b_j - y_i + b_i &\geq y_m - b_m - y_n + b_n &&\forall m,n \in [d+1]   \label{eq61}
\end{align}

To see that $y$ satisfies \eqref{eq41}, observe that $y_\ell - y_k = x_i - x_j$ and $y_j - b_j - y_i + b_i = x_k - b_k - x_\ell + b_\ell$. These two expressions are then equal by \eqref{eq1}.

To see that it satisfies \eqref{eq51}, observe that $y_\ell - y_k = x_i - x_j \geq x_m - x_n = y_m - y_n$ for all $m,n \in [d+1] \backslash \{i,j\}$ so it suffices to show that 
\begin{align*} 
&y_\ell - y_k = x_i - x_j \geq y_i - y_j = x_\ell - b_\ell - x_k + b_k + b_i - b_j \\
&y_\ell - y_k = x_i - x_j \geq y_j - y_i = x_k - b_k - x_\ell + b_\ell + b_j - b_i.
\end{align*}
To see the first inequality, observe that from \eqref{eq31} $x_k - b_k - x_\ell + b_\ell \geq x_j - b_j - x_i + b_i$. Multiplying by $-1$ and rearranging the terms then gives the first inequality. To see the second inequality, observe that from \eqref{eq11}, $x_i - x_j = x_k - b_k - x_\ell + b_\ell$. In addition, by \Cref{lem: bis non empty bijkl order} we have that $b_i \geq b_j$ so $b_j - b_i \leq 0$. Combining these two statements then gives $x_i - x_j \geq x_k - b_k - x_\ell + b_\ell + b_j - b_i$.

Finally, to see that $y$ satisfies \eqref{eq61}, first observe that
\begin{align*} 
y_j - b_j - y_i + b_i &= x_k - b_k - x_\ell + b_\ell \\
&\geq x_m - b_m - x_n + b_n \ \forall m,n \in [d+1] \\
&= y_m - b_m - y_n + b_n \ \forall m,n \in [d+1] \backslash \{m,n\}.
\end{align*}
Therefore, it suffices to show that 
\begin{align*}
&y_j - b_j - y_i + b_i = x_k - b_k - x_\ell + b_\ell \geq y_k - y_\ell = x_j - x_i, \\
&y_j - b_j - y_i + b_i = x_k - b_k - x_\ell + b_\ell \geq y_\ell - y_k = x_i - x_j.
\end{align*}
To see why the first inequality holds, observe that from \eqref{eq31} $x_k - b_k - x_\ell + b_\ell \geq 0$ and from \eqref{eq21}, $x_j - x_i \leq 0$. Therefore, $x_k - b_k - x_\ell + b_\ell \geq 0 \geq x_j - x_i$. The second inequality holds from \eqref{eq11}.

Observe that this affine transformation is invertible, giving that $\Sijkl (b) \neq \emptyset$ if and only if $S_{\ell,k,j,i} (b) \neq \emptyset$. 
Finally, from \Cref{lem: bis non empty bijkl order} we have that if $\Sijkl (b) \neq \emptyset$, then $b_j \leq b_\ell$.
\end{proof}

\Cref{lem: pairs of nonempty} shows that certain components of $\bis (b)$ come in pairs. We now show that certain combinations of the linear systems \eqref{eq: bis sys} cannot be simultaneously feasible.

\begin{lemma}\label{lem: sijjk or slikl}
    Let $b \in \trt$ be generic and let $i,j,k,\ell \in [d+1]$ where $b_j = \min \{b\}$, $b_{\ell} = \max \{b\}$, $i \neq j$ and $k \neq \ell$. Then at most one of $S_{i,j,j,k}(b)$ or $S_{\ell, i,k,\ell}(b)$ can be non-empty.
\end{lemma}
\begin{proof}
Let $b \in \trt$ be generic and without loss of generality suppose that $b_1 = \min \{b\}$ and $b_{d+1} = \max \{b\}$. For the sake of contradiction, suppose that there exists some $m,n \in [d+1]$ such that $S_{m,1,1,n}(b) \neq \emptyset$ and $S_{d+1,m,n,d+1}(b) \neq \emptyset$. From \Cref{lem: pairs of nonempty} we have that $S_{m,1,1,n} (b) \neq \emptyset$ if and only if $S_{n,1,1,m}(b) \neq \emptyset$, therefore without loss of generality, assume that $b_n < b_m$.

Since $S_{m,1,1,n} \neq \emptyset$, there exists some $x \in \trt$ such that
\begin{align}
    x_m - x_1 &= x_1 - x_n - b_1 + b_n \label{eqq1}\\
    x_m - x_1 &\geq x_m - x_n \label{eqq2}\\
    x_m - x_1 &\geq x_{d+1} - x_1 \label{eqq3}\\
    x_1 - x_n - b_1 + b_n &\geq x_1 - x_{d+1} - b_1 + b_{d+1}. \label{eqq4}
\end{align}
 Rearranging \eqref{eqq4} gives that $-x_n + x_{d+1} \geq - b_n + b_{d+1}$. Equations \eqref{eqq2} and \eqref{eqq3} implies that $x_m - x_1 \geq -x_n + x_{d+1} $. Combining these inequalities gives that $x_m - x_1 \geq -b_n + b_{d+1}$. The equality in \eqref{eqq1} then gives that
 $x_1 - x_n - b_1 + b_n \geq - x_n + x_{d+1}$. This implies 
 \[ b_1 - b_n + b_{d+1} \leq x_1 - x_n + b_n \leq b_n, \]
where the last inequality comes from \eqref{eqq2}.

Now since $S_{d+1,m,n,d+1} (b) \neq \emptyset $ there exists some $y \in \trt$ such that 
\begin{align*}
    y_{d+1} - y_m &\geq y_n - y_1 \\
    y_n - y_{d+1} - b_n + b_{d+1} &\geq y_1 - y_m - b_1 + b_m. 
\end{align*}
This then gives that 
\[ y_1 - y_m - b_1 + b_m \leq y_n - y_{d+1} - b_n + b_{d+1} \leq y_1 - y_m - b_n + b_{d+1}.\]
Simplifying this expression implies $b_m \leq b_1 - b_n + b_{d+1}$. Combining these results gives that $b_m \leq b_n$ which is a contradiction. 
\end{proof}

We organize the preceding lemmas into the following theorem.
\begin{theorem}\label{thm: big thm}
    \begin{enumerate}[label = (\alph*)]
    \item[] 
        \item If $S_{i,j,k,\ell}(b) \neq \emptyset$ then $b_j \leq b_i$, $b_k \leq b_\ell$,  $b_k\leq b_i$ and $b_j \leq b_\ell$. 
        \item\label{thm b} If $S_{i,j,j,\ell}(b) \neq \emptyset$ then $b_j = \min 
        \{b\}$ 
        and $b_j \leq 2b_\ell - b_i$.
        \item\label{thm c} If $S_{i,j,k,i}(b) \neq \emptyset $ then $b_i = \max \{b\} $ 
        and $b_j \geq 2 b_k - b_i$.
        \item If $b \in \trt$ is generic, $b_j = \min \{b\}$ and $b_i = \max \{b\}$ then at most one of $S_{m,j,j,n}(b)$ and $S_{i,m,n,i}(b)$ is non-empty.
    \end{enumerate}
\end{theorem}

We now use \Cref{thm: big thm} to give information on the geometry of $\bis(b)$. We begin by presenting the main theorem of this section, an upper bound on $\mathcal{C}(b)$. 

\begin{theorem}\label{cor: bound}
   For generic $b \in \trt$, $\mathcal{C}(b) \leq 4 \cdot \binom{d+1}{4} +  2 \cdot \binom{d+1}{3} + 2 \cdot \binom{d}{2}  + 1$. \label{lem:last-bound}
\end{theorem}

\begin{proof}
Since $b \in \trt$ is generic, we assume all coordinates of $b$ are distinct and we count how many $S_{i,j,k,\ell}(b)$ can be non-empty. We consider the following mutually exclusive cases based on how many indices of $S_{i,j,k,\ell}(b)$ are distinct.

\begin{enumerate}
    \item Four indices are distinct. 
    
\item[] Without loss of generality, assume $b_i<b_j<b_k<b_\ell$. In this case, of the $24$ $S_{i,j,k,\ell}(b)$ with distinct indices, \Cref{thm: big thm} gives that there are four that are potentially non-empty: $S_{k,i,j,\ell},S_{\ell,i,j,k},S_{k,j,i,\ell},S_{\ell,j,i,k}$. This gives a total of 
\[ 4 \cdot \binom{d+1}{4}\]
potential non-empty linear components.
\item Three indices are distinct.
\item[] Without loss of generality, assume $b_i < b_j < b_k$. In this case, of the $24$ $S_{i,j,k,\ell}(b)$ with three distinct indices, \Cref{thm: big thm} gives that there are at most $6$ non-empty: 
\[S_{j,i,i,k}(b), \ \ S_{k,i,i,j}(b),\]\[S_{j,i,j,k}(b), \ \ S_{k,j,i,j}(b),\]\[  S_{k,i,j,k}(b), \ \ S_{k,j,i,k}(b).\]

The top pair $S_{j,i,i,k}(b), S_{k,i,i,j}(b) $ are both empty or both non-empty by \Cref{lem: pairs of nonempty}, and by \Cref{lem:max and min}, both are only non-empty only if $b_i = \min\{b\}$. There are $2\cdot {d \choose 2}$ remaining possible choices of $b_j$ and $b_k$.
Similarly for the bottom pair, $S_{k,i,j,k}(b)$ and $S_{k,j,i,k}(b)$ are simultaneously empty or non-empty and they are non-empty only if $b_k = \max \{b\}$. There are $2\cdot {d\choose 2}$ of these remaining choices as well. By \Cref{lem: sijjk or slikl}, the top and bottom pairs are mutually exclusive, as we have that at most one of $S_{j,i,i,k}(b)$ and $S_{k,i,j,k}(b)$ can be non-empty. This gives a total of $2 \cdot \binom{d}{2}$ potential non-empty linear components from the top and bottom row. 
 
We have no results restricting $S_{j,i,j,k}(b)$ or $S_{k,j,i,j}(b) $ based on the previous sets or each other, so we count the total ${d+1 \choose 3}$ possible choices of $b_i<b_j<b_k$ for each. Therefore, when three of the indices for $S_{i,j,k,\ell}(b)$ are distinct, there are at most 
\[ 2 \cdot \binom{d+1}{3} + 2 \cdot \binom{d}{2} \] non-empty linear components.

\item Two indices are distinct.
\item[] Without loss of generality, assume $b_i < b_j$. Then the only potentially non-empty linear system is $S_{j,i,i,j}(b)$. By \Cref{lem:max and min}, this will be non-empty if and only if $b_i = \min\{b\}$ and $b_j = \max\{b\}$. Therefore, this case adds exactly one linear component to $\bis(b)$. 
\item[] 
\item[] Adding together all of the above cases gives that 
\[ \mathcal{C}(b) \leq 4 \cdot \binom{d+1}{4} +  2 \cdot \binom{d+1}{3} + 2 \cdot \binom{d}{2}  + 1 \]
\[ = \frac{d^4}{6}+\frac{5d^2}{6} - d + 1.\]
\end{enumerate}

\end{proof}

Finally we note that $\bis(b)$ generically will have an odd number of components.

\begin{corollary}
    For generic $b$, $\mathcal{C}(b)$ is odd.
\end{corollary}
\begin{proof}
    From \Cref{lem: pairs of nonempty}, we have that so long as $i \neq \ell$ or $j \neq k$, that number of linear components come in pairs. Then from \Cref{lem:max and min}, we get a single component.
\end{proof}

Based on computational results in the following section, which achieve the upper bound of \Cref{cor: bound} for a subset of generic $b$ vectors in all dimensions tested, we conjecture that the bound in \Cref{cor: bound} is tight.

\subsection{Computational Results}

\begin{example}
    Consider $d+1 = 3$ and generic $b$ with $b_1 < b_2 < b_3$, then the following matrix represents the only possible sets $S_{i,j,k,l} (b) \neq \emptyset$ where $ \times$ denotes sets that are empty and $\star$ represents sets that are possibly nonempty.
   \[ \begin{array}{c|cccccc}
         (i,j) \slash (k,l) & 12 & 21 & 13 & 31 & 23 & 32  \\ \hline  
        12 & \times & \times & \times & \times & \times & \times \\
        21 &  \times &  \times & \star &  \times & \star &  \times \\
        13 & \times & \times & \times & \times & \times & \times \\
        31 & \star &  \times & \star &  \times & \star &  \times \\
        23 &  \times &  \times &  \times &  \times &  \times &  \times \\
        32 & \star &  \times & \star & \times & \times & \times
    \end{array} \]

From \Cref{thm: big thm} \ref{thm b} and \ref{thm c} we have there cannot be a $\star$ in entries $\{(31),(12)\}$ and $\{(31),(23)\}$. Therefore, in this situation, there are two different types of generic behavior, depending on if  $b_1 > 2b_2 - b_3$ (left table) or if $b_1 < 2b_2 - b_3$ (right table).
   \[ \begin{array}{c|cccccc}
         (i,j) \slash (k,l) & 12 & 21 & 13 & 31 & 23 & 32  \\ \hline  
        12 & \times & \times & \times & \times & \times & \times \\
        21 &  \times &  \times & \times &  \times & \star &  \times \\
        13 & \times & \times & \times & \times & \times & \times \\
        31 & \times &  \times & \star &  \times & \star &  \times \\
        23 &  \times &  \times &  \times &  \times &  \times &  \times \\
        32 & \star &  \times & \star & \times & \times & \times
    \end{array} \quad \begin{array}{c|cccccc}
         (i,j) \slash (k,l) & 12 & 21 & 13 & 31 & 23 & 32  \\ \hline  
        12 & \times & \times & \times & \times & \times & \times \\
        21 &  \times &  \times & \star &  \times & \star &  \times \\
        13 & \times & \times & \times & \times & \times & \times \\
        31 & \star &  \times & \star &  \times & \times &  \times \\
        23 &  \times &  \times &  \times &  \times &  \times &  \times \\
        32 & \star &  \times & \times & \times & \times & \times
    \end{array} \]

In either case, we see that the tropical bisector of two generic points in $\mathbb{R}^3 \slash \mathbb{R} \mathbf{1}$ has five linear segments. Observe that that in this case, the bound given by \Cref{cor: bound} is tight.
\end{example}

\begin{example}
    Consider $d + 1 = 4$ and without loss of generality take $b_1 < b_2 < b_3 < b_4$. In this case $S_{i,j,k,\ell} = \emptyset$ if $(i,j) \in \{(1,2),(1,3),(1,3),(2,3),(2,4),(3,4) \}$ or $(k, \ell) \in \{(2,1),(3,1),(3,2),(4,1),(4,2),(4,3)\}$. This then gives us $36$ options. Using \Cref{cor: bound} we have that there are at most $19$ components. In computation we observed that for generic $b \in \trt$, there were either $17$ or $19$ components. When there are $17$ components, the linear components come from the following systems:

      \[ \begin{array}{c|cccccc}
         (i,j) \slash (k,l) & 12 & 13 & 23 & 14 & 24 & 34  \\ \hline  
        21 & \times & \times & \star & \times & \star & \times \\
        31 &  \times &  \times & \times &  \star & \star &  \star \\
        32 & \star & \times & \times & \star & \times & \star \\
        41 & \times &  \star & \star &  \star & \star &  \times \\
        42 &  \star &  \star &  \times &  \star &  \times &  \times \\
        43 & \times &  \star & \star & \times & \times & \times
    \end{array} \]

When there are $19$ components, then two situations occur:

      \[ \begin{array}{c|cccccc}
         (i,j) \slash (k,l) & 12 & 13 & 23 & 14 & 24 & 34  \\ \hline  
       21 & \times & \star & \star & \star & \star & \times \\
        31 &  \star &  \times & \times &  \star & \star &  \star \\
        32 & \star & \times & \times & \star & \times & \star \\
        41 & \star &  \star & \star &  \star & \times &  \times \\
        42 &  \star &  \star &  \times &  \times &  \times &  \times \\
        43 & \times &  \star & \star & \times & \times & \times
    \end{array} \quad \begin{array}{c|cccccc}
         (i,j) \slash (k,l) & 12 & 13 & 23 & 14 & 24 & 34  \\ \hline  
        21 & \times & \times & \star & \times & \star & \times \\
        31 &  \times &  \times & \times &  \times & \star &  \star \\
        32 & \star & \times & \times & \star & \times & \star \\
        41 & \times &  \times & \star &  \star & \star &  \star \\
        42 &  \star &  \star &  \times &  \star &  \times &  \star \\
        43 & \times &  \star & \star & \star & \star & \times
    \end{array} \]
In any case, we observe the following systems are always nonempty:
      \[ \begin{array}{c|cccccc}
         (i,j) \slash (k,l) & 12 & 13 & 23 & 14 & 24 & 34  \\ \hline  
        21 &   &   & \star &   & \star &   \\
        31 &   &    &   &    & \star &  \star \\
        32 & \star &   &  & \star &   & \star \\
        41 &   &    & \star &  \star &    &   \\
        42 &  \star &  \star &   &    &   &   \\
        43 &  &  \star & \star &  &  &  
    \end{array} \]
We observe again that our bound in \Cref{cor: bound} is tight.

 \end{example}

In \Cref{tab:experiment} and \Cref{fig:experiment-bar} we give computational results on the number of times the bisector has $X$ components when sampling each coordinate of $b \in \trt$ as $b_i \sim \mathcal{N}(0,1)$ in $1000$ trials.

\begin{table}[h!]
    \centering
    \tiny
     \begin{tabular}{|c|c|c|c|c|c|c|c|c|c|}
        \hline  $d+1$ & 3 & 4 & 5 & 6 & 7 & 8 & 9 & 10 & 11 \\ \hline 
        $X$& $5 : 100$  & \shortstack{$17: 47.9 $ \\ $19: 52.1$} & \shortstack{$49 : 68$ \\ $53 : 32$} & \shortstack{$113 : 32.4$ \\ $115 : 50$ \\ $121 : 17.6$} & \shortstack{$229 : 57$ \\ $233 : 32.6$ \\ $241 : 10.4$} & \shortstack{$417 : 27.6$ \\ $419 : 44.3$ \\ $425 : 22.1$ \\ $435 : 6$}&  \shortstack{$705 : 47.5$ \\ $709: 32$ \\ $717:16$ \\ $729: 4.5$} & \shortstack{$1121 : 21$ \\ $1123 : 39.4$ \\ $1129: 24.9$ \\ $1139: 10.9$ \\ $1153: 3.8$} & \shortstack{$1701: 37.6$ \\ $1705 : 32.7$ \\ $1713 : 18.8$ \\ $1725 : 8.6$ \\ $1741 : 2.3$} \\ \hline 
        \shortstack{Bound in  \\Thm. \ref{cor: bound}} & 5 & 19 & 53 & 121 & 241 & 435& 729 & 1153 & 1741\\ \hline
    \end{tabular}
    \caption{The percentage of times the bisector has $X$ components when $b \in \mathbb{R}^d$ where each $b_i \sim \mathcal{N}(0,I)$ is independently sampled in $1000$ trials and the bound derived in \Cref{cor: bound}.}
     \label{tab:experiment}
\end{table}

\begin{figure}
    \centering
    \includegraphics[width=0.48\linewidth]{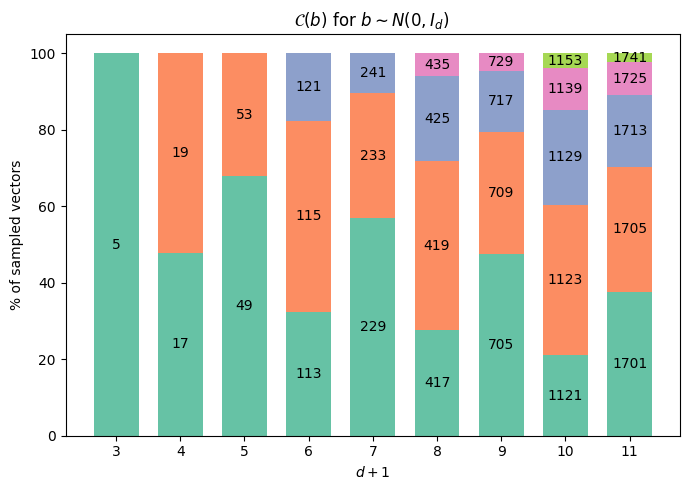}
        \caption{The distribution of $\mathcal{C}(b)$ for 1000 independently sampled $b \in \mathbb{R}^d$, $b_i \sim \mathcal{N}(0,I)$. The experimental data achieves the upper bound from \Cref{lem:last-bound} in a decreasing minority of cases as $d$ increases. We note that even when the upper bound is not realized, $\mathcal{C}(b)$ has a relatively small range of values for generic $b$.}
    \label{fig:experiment-bar}
\end{figure}

\subsection{Examples of tropical bisectors and their interaction with the gradient}

Inspired from the proof of \cite[Proposition 4]{CJS}, we have implemented an algorithm in SageMath \cite{sagemath}, that plots the bisector of any two given points in the plane. For examples of such plots see \Cref{fig:minus1.5} and \Cref{fig:w2}.

\begin{figure}[h]
    \centering
    \includegraphics[height=4cm]{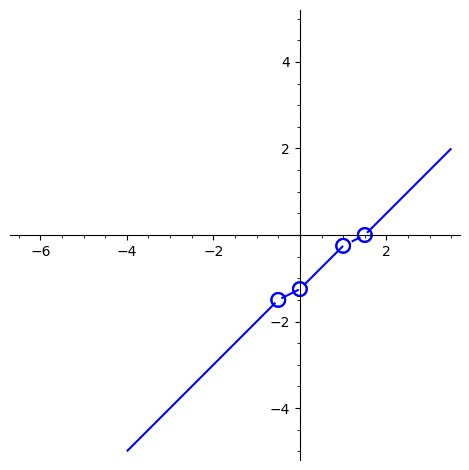}
    \caption{Bisector of two points in the plane: $(0,0,0)$ and $(1,-1.5,0)$.}
    \label{fig:minus1.5}
\end{figure}

\begin{figure}[h]
    \centering
    \includegraphics[height=4cm]{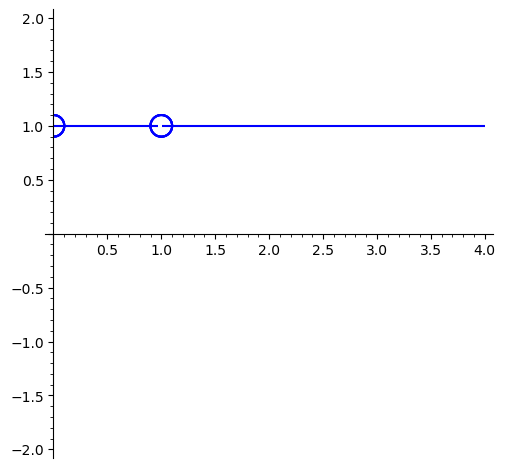}
    \caption{Bisector of two points in the plane: $(0,0,0)$ and $(1,2,0)$.}
    \label{fig:w2}
\end{figure}

Moreover, using this algorithm we can see 
how the tropical bisector interacts with the gradient of the Carlini-Wagner attack. An example of this is given in \Cref{fig:enter-label}. More precisely, given two classes centered at the purple and black points, 
we see the piecewise linear tropical bisector.  
Further, we suppose that the blue point is the starting point of the Carlini-Wagner attack. The top left of figure in \Cref{fig:enter-label} illustrates Carlini-Wagner gradients using the objective function $\| \delta \|_2 + f(x + \delta)$ compared to the tropical bisector. Here we define 
\[ f(x) = \max \{ 0, (\max_{i \neq t} Z(x)_i) - Z(x)_t \} \]
where $Z(x)$ is the output of all layers of the tropical CNN except the softmax layer.  
In the top right figure in \Cref{fig:enter-label} we show Carlini-Wagner gradients using $d_\Trop(\delta) + f (x + \delta)$. Notice that the gradient moves toward the bisector, but does not vanish immediately upon crossing it in some segments. On the bottom of \Cref{fig:enter-label}, we show a close-up view of the gradient near the decision boundary.

\begin{figure}[h]
    \centering
    \includegraphics[width = 0.49\linewidth]{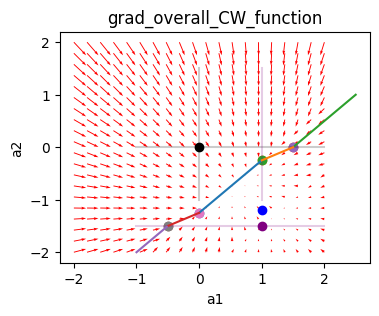}\includegraphics[width=0.49\linewidth]{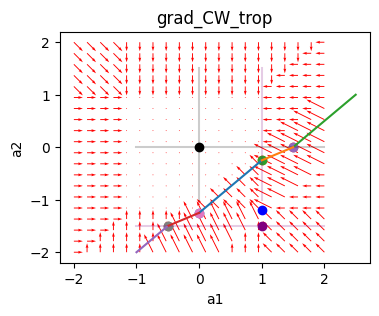}
    \includegraphics[width=0.49\linewidth]{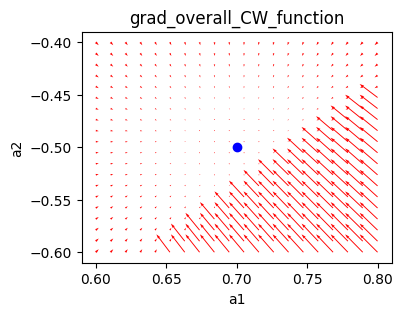}
    \includegraphics[width=0.49\linewidth]{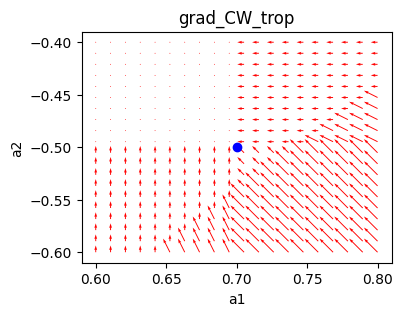}
    \caption{Interaction of bisector with the gradient. }
    \label{fig:enter-label}
\end{figure}

\section{Carlini-Wagner Attacks}
The study conducted by \cite{Pasque2024} indicates that incorporating a tropical layer as defined in \eqref{def:trop_embedding_layer} can improve the robustness of a convolutional neural network against various state-of-the-art attacks, such as the Carlini-Wagner attack \cite{carlini2017evaluating}. This section investigates why the tropical layer improves robustness, and based on our results, suggests a refined version of the Carlini-Wagner attack, specifically targeted to the tropical architecture.
Before going into more details, we provide a short summary on the CW attack and the tropical network, and introduce some notation.
\subsection{Tropical CNNs}
By a tropical CNN we refer to the combination of a classic CNN, called base model henceforth, plus a tropical layer on top, directly preceding the soft-max layer. More precisely, let $x$ denote the input vector, and for a tropical CNN with $n$ layers, we call the ($n$-th) tropical embedding layer $F_n$, and the (pre-softmax) output of the "classical" CNN $Z_{n-1}(x)$. The prediction of the overall tropical network is their convolution, that is
$$\hat{y} = \text{softmax} \circ F_n \circ Z_{n-1}(x)$$
In the following, we write $Z = Z_n = F_n \circ Z_{n-1}.$
\subsection{CW attack}
Let $c$ be the true label of input vector $x$. We are interested in finding an adversarially disturbed version $x'$; that is, $x'$ close to $x$ such that $x'$ is unlikely to be classified as $c$\footnote{
In other words, we conduct an \textit{untargeted} attack, aiming to classify $x'$ as any class other than $c$. The Carlini-Wagner attack method also supports \textit{targeted} attacks, which aim to misclassify $x'$ into a specific predefined class $t \neq c$. This is achieved by replacing the function $f$ defined in \eqref{eq:cw_attack} with $\max(\max_{i\neq t}(Z(x^{\prime})_{i})-Z(x^{\prime})_{t},-\kappa)$. For simplicity, we focus on the untargeted attacks.}. Carlini and Wagner suggest to minimize 
\begin{align*}
    g(x') = \|x'-x\|_{2}^{2}+\lambda\cdot f(x')
\end{align*}
for a constant $\lambda>0$ and $f$ some function being below zero if and only if $x'$ is misclassified. They propose six options for $f$, but focus on $f$ defined as
\begin{align}\label{eq:cw_attack}
f(x^{\prime})=\max\left(Z(x^{\prime})_{c}-\max_{i\neq c}(Z(x^{\prime})_{i}),-\kappa\right)
\end{align}
since this choice worked best in their simulations. 
In minimizing $g$, we simultaneously minimize the distance between $x'$ and $x$ and $f$. This rewards an increase in the maximal $Z(x')_i$ relative to $Z(x')_t$, maximizing the likelihood that $x'$ will be classified as $i$ rather than $c$ under softmax. 
For details on why $Z$ and not the post-softmax output is used, and on the purpose of the hyperparameter $\kappa\geq 0$, we refer to the original paper. As in the simulation part of the original paper, we focus on the choice $\kappa=0$ in the following.

Finally, to ensure that $x'$ is indeed a valid image in the sense that each pixel lies in the range $[0, 1]$, they introduce a change of variables $x'=\frac{1}{2}(\tanh(z)+1)$, and find the $z^\star$ minimizing
\begin{align}\label{eq:cw_attack_variable_change}
    g(z) = \left\|\frac{1}{2}(\text{tanh}(z)+1)-x\right\|_{2}^{2}+\lambda\cdot f\left(\frac{1}{2}(\text{tanh}(z)+1)\right).
\end{align}
\subsection{Reason for robustness}
To find the optimum $z^\star$ of $g$, gradient descent is used. While this works very well for most standard CNNs, for tropical CNNs the success rate of CW attacks drops significantly \cite{Pasque2024}. In this section we explain why this is the case: the gradient of a tropical CNN is not continuously differentiable, and in fact depends on a relatively low proportion of input coordinates. This may produce oscillatory effects and obstructs efficient optimization.  
As a simple example, we first consider minimizing
\begin{equation*}
    \max(h(x))
\end{equation*} in $x$ for some continuously differentiable $h: \mathbb{R}^m \to \mathbb{R}^k$.
Suppose we start gradient descent at some $x_0$ with current function value $\max(h(x_0))=y$, and, without loss of generality, assume this maximum occurs at the first coordinate, that is, $h_1(x_0)=y$. At the same time, $h_2(x_0)$ might be only slightly smaller, specifically $h_2(x_0) = y-\delta$. Nevertheless, by chain rule,
\begin{equation*}
    \nabla (\max(h(x_0))) = (e_1^T J_h(x_0))^T = \nabla h_1(x_0),
\end{equation*}
which ignores all coordinate functions $h_i$ except for $h_1$. As a result, the gradient descent moves in a direction that reduces $h_1$, while not considering the values of the other coordinate functions of $h$. In particular, $h_2$ could increase and now be slightly larger than the reduced $h_1$. Hence, in the second step, the gradient descent leads to a decrease $h_2$ and again ignores all other components, so $h_1$ could increase again. This back-and-forth can cause the process to oscillate between points $x_n, x_{n+1}$ where $h_1$, $h_2$, are slightly better respectively, but the overall maximum might barely improve. For instance, consider
\begin{equation}\label{ex:oscillation_GD}
    h(x) = \begin{pmatrix} x_1^2 + 2\cdot x_2 \\
    x_1^2 - 2.2\cdot x_2 
    \end{pmatrix}.
\end{equation} Then, $\max(h(x))$ has gradient
\begin{equation*}
    \nabla (\max(h(x))) = \begin{pmatrix}
        2x_1 \\
         2 \text{ if } x_2 > 0,\ -2.2 \text{ otherwise} 
    \end{pmatrix}.
\end{equation*}
and a unique minimum at $(0,0)$. Far from the minimum, gradient descent proceeds correctly. However, close to $(0,0)$, the first gradient component becomes negligible compared to the second, which alternates between 2 and -2.2. This stark difference in scale induces oscillations, as visualized in  \Cref{fig:oscillation_GD}. 
\begin{figure}
    \centering
    \includegraphics[width=0.6\linewidth]{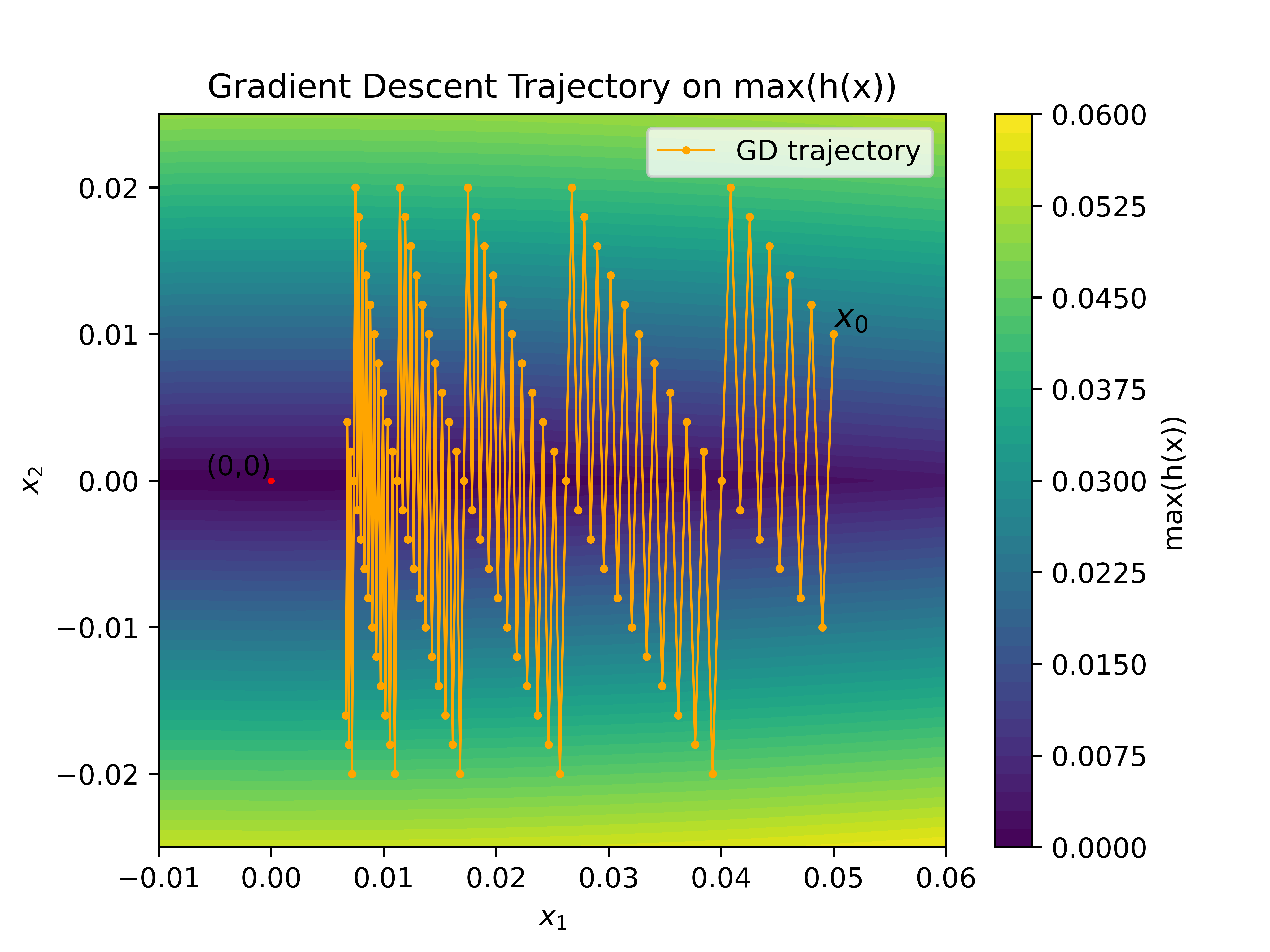}
    \caption{The first 100 iterations of GD on $\max(h(x))$. Due to oscillation,  GD does not converge within these iterations.}
    \label{fig:oscillation_GD}
\end{figure}
Especially in high dimensional image spaces, it is very likely to encounter such an oscillation: for minimizing $\max(h(x))$ for $h: \mathbb{R}^m \to \mathbb{R}^k$ in effect all $k$ component functions need to be minimized, but the gradient descent only decreases one component function in each step and does not control what happens to the other components. See \cite[Section VI.C.]{carlini2017evaluating} for another example; they encountered a similar problem when they replaced the $\ell_2$ norm in \eqref{eq:cw_attack_variable_change} by the maximum norm.

Returning to the tropical CNN, we seek to minimize
\begin{align*}
  f(x) =&\;  d_{tr}(Z_{n-1}(x)-w_{c}) - \max_{i\neq c}(d_{tr}(Z_{n-1}(x)-w_{i})) \\ =& \max((Z_{n-1}(x)-w_{c}) - \min(Z_{n-1}(x)-w_{c}) \\ &- \max_{i\neq c}(\max(Z_{n-1}(x)-w_{i}) + \min(Z_{n-1}(x)-w_{i}))
\end{align*} plus the $\ell_2$ penalization. Now, there occurs not only one maximum but five maxima/minima. Note that the maximum over the classes is usually no problem since this maximum is taken over only a few classes. In contrast, the other four minima/maxima are likely to lead to oscillation. Here, the maximum/minimum is taken over all coordinates of $Z_{n-1}(x)-w_{i}$, resulting in a very sparse gradient in $\mathbb{R}^{\text{length}(Z_{n-1}(x))}$ with only four non-zero entries.

To reveal whether this phenomenon indeed contributes to the robustness found in \cite{Pasque2024}, and if so, how much of the robustness it explains, we experiment with a slightly altered version of gradient descent. 
\subsection{Improved attack by modified gradient descent}
This altered version aims at addressing the issue of considering only four coordinates. Correspondingly, we do not only look at the four coordinates achieving the maxima and minima, but at all coordinates exhibiting values close to the maximum, similar to a suggestion in \cite{carlini2017evaluating}. Specifically, we introduce a threshold $\tau>0$, and replace each maximum $\max(z)$ in the tropical distance by
\begin{equation}
    \sum_{i=1}^d (z_i - \tau)^+,
\end{equation}
where $(x)^+=\max(x, 0)$.
This penalizes all components exceeding the threshold, not just the largest one. By selecting $\tau$ appropriately, both the largest component and those with values close to it will be penalized. Therefore, we control multiple rather than only one component, which ensures a more efficient optimization. 

Similarly, denoting $(x)^-=\min(x, 0)$, we replace minima  by
\begin{equation}
\sum_{i=1}^d (z_i + \tau)^-,
\end{equation} leading to the updated tropical distance
\begin{align*}
\tilde{d}_{\text{tr}}(x, w) = \sum_{i=1}^d (x_i - w_i - \tau)^+ - \sum_{i=1}^d (x_i - w_i + \tau)^-.
\end{align*}
Then, we substitute each occurrence of the tropical distance in $g$ by $\tilde{d}_{\text{tr}}(x, w)$ and compare how this affects the robustness. 

 \Cref{tab:altered_gradient_descent} and \Cref{tab:altered_gradient_descent2} summarize our results on four different networks. Specifically, we use 
 LeNet5 and ModifiedLeNet5 as our base networks and then add either a fully  connected linear layer or a tropical layer as the last layer. All four networks reach a test accuracy above $99\%$ on the MNIST data set. We include the attack with the normal vs. our modified gradient descent, and report the percentage of successful attacks along with the mean $\ell_2$ distances between the adversarial and the original image. Some of the cells remain empty since using the altered gradient descent for a linear top layer does not make sense; this attack is specifically designed against tropical networks. 
 For the precise choice of the networks and the other hyperparameters we refer to the Appendix \ref{subsec:networks_and_hyperparams}. Figure \ref{fig:examples_MNIST} shows some pictures of the MNIST data set along with attacks computed with the altered gradient descent. 
\renewcommand{\arraystretch}{1.5}
\begin{table}[h]
\centering
\begin{tabular}{|c|c|c|c|c|c|}
    \hline
    \multirow{2}{*}{Base model} & \multirow{2}{*}{Top layer} & \multicolumn{2}{c|}{Gradient descent} & \multicolumn{2}{c|}{Altered gradient descent} \\ \cline{3-6}
        & & Attack successrate & Mean $\ell_2$ & Attack successrate & Mean $\ell_2$ \\ \hline
    \multirow{2}{*}{ModifiedLeNet5} & Linear & 97\% & 1.63  & - & -   \\ \cline{2-6}
        & Tropical & 42\% & 1.82 & 63\% & 1.99 \\ \hline
    \multirow{2}{*}{LeNet5} & Linear & 97\% &  2.08 & - & - \\ \cline{2-6}
        & Tropical & 61\% & 2.51 & 72\% & 2.28 \\ \hline
\end{tabular}
\caption{The altered gradient descent increases the attack sucess rate.} 
\label{tab:altered_gradient_descent}
\end{table}
\renewcommand{\arraystretch}{1}

\begin{figure}
    \centering
    \includegraphics[width=0.99\linewidth]{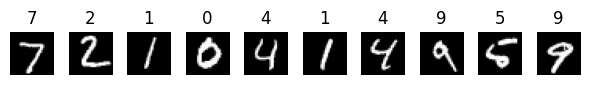}
    \includegraphics[width=0.99\linewidth]{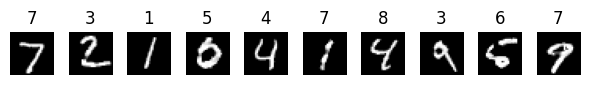}
    \caption{Original images (top) vs. adversarial images (bottom). The number above each image is the prediction of a tropical network with base model LeNet5.}
    \label{fig:examples_MNIST}
\end{figure}

\subsection{Further improvement through multiple starting point gradient descent}
Another way to improve the suggestion made in \cite{carlini2017evaluating} is the so-called multiple starting point gradient descent: normally, when computing the adversarial image, one starts the gradient descent at $x_0 = x$, $x$ being the original image. As an alternative, one can also start at $x_0 = x + \delta$ with a random noise $\delta$ uniformly chosen from $\partial B_r(0)$ for some small $r$. As the name indicates, in practice, one does not only pick one, but multiple random starting points, generating an attacked image from each, and then selecting the best result. We use $10$ such points. Note that this strategy can be combined with the classical as well as with our modified gradient descent since it only affects the choice of $x_0$. By implementing this suggestion, we could improve the attack success rate even more, as shown in the following table.
\renewcommand{\arraystretch}{1.5}
\begin{table}[h]
\centering
\begin{tabular}{|c|c|c|c|c|c|}
    \hline
    \multirow{2}{*}{Base model} & \multirow{2}{*}{Top layer} & \multicolumn{2}{c|}{Gradient descent + MSP} & \multicolumn{2}{c|}{Altered gradient  descent + MSP} \\ \cline{3-6}
        & & Attack successrate & Mean $\ell_2$ & Attack successrate & Mean $\ell_2$ \\ \hline
    \multirow{2}{*}{\parbox{1.8cm}{\centering Modified-\newline LeNet5}} & Linear & 100\% & 1.25  & - &  -  \\ \cline{2-6}
        & Tropical & 59\% & 2.05 & 87\% & 1.59 \\ \hline
    \multirow{2}{*}{LeNet5} & Linear & 100\% &  0.94 & - & - \\ \cline{2-6}
        & Tropical & 73\% & 1.06 & 98\% & 1.26 \\ \hline
\end{tabular}
\caption{Using multiple starting points (MSP), increases the successrate even further. For LeNet5, the altered gradient descent now succeeds almost always.} 
\label{tab:altered_gradient_descent2}
\end{table}

\section*{Acknowledgements}
R. Yoshida is partially supported by the NSF Statistics Program DMS 2409819. 
M.-{\c S}. Sorea is supported by the project ``Mathematical Methods and Models for
Biomedical Applications'' financed by National Recovery and Resilience Plan PNRR-III-C9-2022-I8. D. Schkoda is financed by the European Research Council (ERC) under the European Union’s Horizon 2020 research and innovation programme (grant agreement No 883818) and  supported by the DAAD programme Konrad Zuse Schools of Excellence in Artificial Intelligence, sponsored by the Federal Ministry of Education and Research. G. Grindstaff is supported by EPSRC Centre to Centre Research Collaboration grant EP/Z531224/1.
The authors are grateful to Jane Coons for bringing us together at the University of Oxford during the Workshop for Women in Algebraic Statistics in July 2024. The Workshop was supported by St John's College, Oxford, the L'Oreal-UNESCO For Women in Science UK and Ireland Rising Talent Award in Mathematics and Computer Science (awarded to Jane Coons), the Heilbronn Institute for Mathematical Research, and the UKRI/EPSRC Additional Funding Programme for the Mathematical Sciences.
\bibliographystyle{plain} 
\bibliography{refs}

\section{Appendix}
\subsection{Networks considered}\label{subsec:networks_and_hyperparams}

We examine the well-known LeNet5 architecture, initially introduced by \cite{LeNet5}, with the layers listed in \Cref{tab:lenet5}. In addition, we evaluate a modified version of LeNet5, proposed by \cite{Pasque2024}. Its architecture is detailed in  \Cref{tab:ModifiedLeNet5}
\begin{table}
\centering
\begin{tabular}{|l|c|l|}
\hline
\textbf{Layer}         & \textbf{Activation} & \textbf{Key Parameters} \\ \hline
Convolution            & Tanh               & 6 filters, 5x5 windows, 1x1 strides, padding=2 \\ \hline
Average Pooling        & -                  & 2x2 windows, stride=2                          \\ \hline
Convolution            & Tanh               & 16 filters, 5x5 windows, 1x1 strides          \\ \hline
Average Pooling        & -                  & 2x2 windows, stride=2                          \\ \hline
Flatten                & -                  & -                         \\ \hline
Fully Connected & Tanh               & 120 neurons                                   \\ \hline
Fully Connected & Tanh               & 84 neurons                                    \\ \hline
Fully Connected & ReLU               & 10 neurons                         \\ \hline
\end{tabular}
\caption{Layers of the LeNet5 network}
\label{tab:lenet5}
\end{table}
\begin{table}
\centering
\begin{tabular}{|l|c|l|}
\hline
\textbf{Layer}         & \textbf{Activation} & \textbf{Key Parameters} \\ \hline
Convolution            & ReLU               & 64 filters, 3x3 windows, 1x1 strides \\ \hline
Max Pooling            & -                  & 2x2 windows, non-intersecting        \\ \hline
Convolution            & ReLU               & 64 filters, 3x3 windows, 1x1 strides \\ \hline
Max Pooling            & -                  & 2x2 windows, non-intersecting        \\ \hline
Convolution            & ReLU               & 64 filters, 3x3 windows, 1x1 strides \\ \hline
Flatten                & -                  & -                 \\ \hline
Fully Connected  & ReLU               & 64 neurons                           \\ \hline
Fully Connected & ReLU               & 10 neurons                \\ \hline
\end{tabular}
\caption{Description of ModifiedLeNet5}
\label{tab:ModifiedLeNet5}
\end{table}
Both networks were trained using a learning rate of 0.001, reduced by a factor of 0.1 when validation accuracy failed to improve for 5 consecutive epochs. If the learning rate was already decreased three times and the validation accuracy plateaued again, we stop early. We trained for a maximum of 32 epochs with a batch size of 64.

When computing the attacks, we use $\kappa=0$ and learning rate of $0.1$. The value $c$ is initialized as $c=0.001$ and adjusted using a binary search within the range  $[0, 10^{10}]$: Every 10 steps optimization, we evaluate whether \eqref{eq:cw_attack_variable_change} decreased. If not, and the attack was unsuccessful, $c$ is increased; otherwise, if the attack succeeded, $c$ becomes larger.
The threshold $\tau$ is initially chosen as the 7th highest value in the vector $|x - w_{c}|$ and decreased by $0.9$ each time all  entries of $|x - w_{c}|$ and all entries of $|x - w_i|$, $i \neq c$, fall below this threshold.

\bigskip 

\noindent
\footnotesize 

{\bf \noindent Authors:}

\smallskip

\noindent Gillian Grindstaff\\
University of Oxford, U.K.\\
 {\tt grindstaff@maths.ox.ac.uk}
\vspace{0.25cm}

\noindent Julia Lindberg\\
The University of Texas at Austin, U.S.A.\\
 {\tt julia.lindberg@math.utexas.edu}
\vspace{0.25cm}

\noindent Daniela Schkoda\\
Technical University of Munich, Germany\\
 {\tt daniela.schkoda@tum.de}
\vspace{0.25cm}

\noindent Miruna-\c Stefana Sorea\\ 
Lucian Blaga University of Sibiu, Romania\\
{\tt mirunastefana.sorea@ulbsibiu.ro}
\vspace{0.25cm}

\noindent Ruriko Yoshida\\
Naval Postgraduate School, U.S.A.\\
 {\tt ryoshida@nps.edu}

\end{document}